\documentclass[3p,12pt]{elsarticle}

\usepackage[utf8]{inputenc}
\usepackage[T1]{fontenc}
\usepackage{amsmath, amssymb, amsthm, bm}
\usepackage{graphicx}
\usepackage{booktabs}
\usepackage{hyperref}
\usepackage{algorithm}
\usepackage{algorithmic}
\usepackage{enumitem}
\usepackage{multirow}

\newtheorem{proposition}{Proposition}
\newtheorem{remark}{Remark}

\journal{Expert Systems with Applications}

\begin{document}

\begin{frontmatter}

\title{Variance-Preserving Orthogonal Selection (VPOS): Greedy Feature Selection via Orthogonal Deflation in Principal-Component Loading Space}

\author[1]{Baran Koseoglu\corref{cor1}}
\ead{koseoglubaran@gmail.com}
\cortext[cor1]{Corresponding author.}

\author[2]{Berrin Yanikoglu}
\ead{berrin@sabanciuniv.edu}

\address[1]{Independent Researcher, London, UK}
\address[2]{Sabanci University, Istanbul, Turkey}

\begin{abstract}
We present Variance-Preserving Orthogonal Selection (VPOS), an unsupervised feature-selection method that performs sequential orthogonal deflation in the variance-weighted principal component analysis (PCA) loading space $\mathbf{V}_d\mathbf{\Lambda}_d^{1/2}$. After each feature is selected, its loading direction is projected out of all remaining candidates, so subsequent selections cover complementary directions of the rank-$d$ covariance approximation while returning original variables. We establish rank-reduction guarantees and a determinant-growth interpretation, and distinguish VPOS from greedy selection on raw data, unweighted eigenvector pivoting, Principal Feature Analysis (PFA), and Principal Variable Selection (PVS). Experiments enforce $k\leq d$, tune method-specific parameters on validation observations, and evaluate on unseen outer folds. Across seven labelled benchmarks, VPOS improves held-out normalised reconstruction error over matched PCA without deflation on every dataset, with reductions of 1--78\%. It obtains the lowest mean reconstruction error on Wine, Breast Cancer, and MNIST and is within 1.7\% of the lowest error on CIFAR-10 and HighDim. On CIFAR-10, VPOS is approximately 24$\times$ faster than the closely related PVS baseline while incurring a 1.7\% reconstruction gap. These results establish VPOS as an efficient covariance-coverage method, particularly when correlated high-dimensional data must be represented by a small set of identifiable original variables.
\end{abstract}


\end{frontmatter}

\section{Introduction}
Dimensionality reduction transforms high-dimensional data into a lower-dimensional space by extracting meaningful patterns or selecting the most relevant features. Dimensionality reduction is broadly categorised into feature selection and feature extraction. Feature selection selects a subset of unchanged input features according to predefined evaluation metrics, while feature extraction applies mathematical transformations to the high-dimensional data and transforms it into a compressed, lower-dimensional space while maximising the retention of critical variance. We focus on \textit{unsupervised feature selection}: selecting a discrete subset of original features based solely on the intrinsic variance structure of the data, without access to class labels. This setting excludes supervised methods (wrapper and embedded approaches) that require a target variable, as well as feature extraction methods that produce linear combinations rather than original features. When the number of features $D$ far exceeds the number of observations $N$, data tends to cluster on a low-dimensional manifold \cite{bellman1957dynamic}. Principal component analysis (PCA) recovers that manifold efficiently, but its components are dense linear combinations of all original variables \cite{jolliffe2002pca}. In genomics, finance, and clinical measurement, the goal is not a rotated projection but the identification of a small, interpretable set of original variables — specific genes, assets, or biomarkers — that capture the essential variance structure. Gene-expression microarrays are a canonical example: thousands of genes are measured simultaneously, but only a handful of pathway-level signals drive biological variation, and practitioners need to identify the specific genes to target in follow-up assays \cite{alon1999colon}.

Unsupervised feature selection fills this role, yet it faces a structural difficulty \cite{li2017survey}. Independent filters such as marginal variance and Laplacian Score \cite{he2005laplacian} can select a redundant top-$k$ set under strong multicollinearity, with several chosen features corresponding to the same dominant latent factor while secondary signals are missed. Multivariate methods such as Multi-Cluster Feature Selection (MCFS) \cite{cai2010mcfs} mitigate this problem through joint sparse regression, but they do not explicitly impose sequential orthogonal covariance coverage. We refer to the failure mode of redundant independent ranking as the \textit{clustering trap}.

A separate line of work sidesteps discrete selection by placing $\ell_1$ penalties on PCA loadings \cite{zou2006spca, jolliffe2003modified}. These methods produce sparse component weights, not a discrete feature index set. In applications where practitioners need to physically act on selected variables — constructing a portfolio, running a targeted assay — mixing originals is not acceptable.

We propose \textbf{Variance-Preserving Orthogonal Selection (VPOS)}, an unsupervised feature-selection method that applies sequential orthogonal deflation to the rows of the variance-weighted loading matrix $\mathbf{B}=\mathbf{V}_d\mathbf{\Lambda}_d^{1/2}$. At each step, the feature whose loading row has the largest residual norm is selected, and its direction is projected out of all remaining candidates, so subsequent selections cover complementary covariance directions. Operating on the validation-truncated covariance factor concentrates the selection on the principal variance structure while retaining a discrete subset of the original variables. Among the many possible configurations of greedy orthogonal selection --- operating on raw data, unweighted eigenvectors, normalised columns, or the full covariance --- experiments show that this specific combination of variance weighting, PCA truncation, and validation-based rank selection yields a favourable speed--accuracy tradeoff.

The contributions of this work are:
\begin{enumerate}[leftmargin=*, label=(\roman*)]
    \item The \textbf{Weighted Loading Embedding} $\mathbf{B} = \mathbf{V}_d \mathbf{\Lambda}_d^{1/2}$ is formalised as the natural geometric space for feature diversity measurement, with each row's squared norm equal to the feature's contribution to the $d$-dimensional principal subspace.
    \item The loading-space deflation is shown to admit a formal connection to column-pivoted QR (CPQR) on $\mathbf{B}^{\top}$, yielding determinant-growth and rank-reduction interpretations together with explicit conditions on $k$, $d$, and the retained rank.
    \item The effect of the operating matrix is isolated empirically against greedy orthogonal selection on raw $\mathbf{X}$, on unweighted $\mathbf{V}_d$, Principal Feature Analysis (PFA), and the closely related Principal Variable Selection (PVS) method of Skogholt et al.\ \cite{skogholt2023pvs}.
    \item Experiments across eight datasets use validation-only tuning, enforce $k\leq d$, and evaluate reconstruction and classification on observations excluded from selection and tuning. The design uses nested outer-fold evaluation for every labelled dataset, with five folds for smaller datasets and three folds for MNIST and CIFAR-10.
\end{enumerate}

\section{Problem Formulation}

\subsection{Matrix Representation and Covariance}
We consider a dataset $\mathbf{X} \in \mathbb{R}^{D \times N}$ where $D$ denotes the total number of input features and $N$ denotes the number of observations. Throughout, $k$ is the selection budget---the number of original features returned by VPOS---whereas $d$ is the retained PCA rank used to construct the internal loading representation. We adopt the $D \times N$ convention (features as rows, observations as columns) to align with signal processing and covariance analysis standards \cite{hyvarinen2001ica}. We assume the data is centred such that $\sum_{i=1}^N \mathbf{x}_i = \mathbf{0}$. The global structure of the feature space is encoded in the feature-covariance matrix $\mathbf{\Sigma} \in \mathbb{R}^{D \times D}$:
\begin{equation}
\mathbf{\Sigma} = \frac{1}{N-1} \mathbf{X} \mathbf{X}^\top
\end{equation}
This covariance matrix captures all pairwise linear relationships between features and provides the substrate for the loading embedding that VPOS operates on.

\subsection{The Geometry of the Loading Space}
The spectral decomposition of $\mathbf{\Sigma}$ is given by:
\begin{equation}
\mathbf{\Sigma} = \mathbf{V} \mathbf{\Lambda} \mathbf{V}^\top
\end{equation}
where $\mathbf{V} = [\mathbf{v}_1, \mathbf{v}_2, \dots, \mathbf{v}_D]$ is the orthonormal matrix of eigenvectors (loadings) and $\mathbf{\Lambda} = \text{diag}(\lambda_1, \lambda_2, \dots, \lambda_D)$ is the diagonal matrix of eigenvalues sorted in descending order ($\lambda_1 \ge \lambda_2 \ge \dots \ge \lambda_D \ge 0$).

We define the \textbf{Weighted Loading Embedding} $\mathbf{B} \in \mathbb{R}^{D \times d}$ for the top $d$ principal components as:
\begin{equation}
\mathbf{B} = \mathbf{V}_d \mathbf{\Lambda}_d^{1/2}
\end{equation}
where $\mathbf{V}_d \in \mathbb{R}^{D \times d}$ contains the top-$d$ eigenvectors and $\mathbf{\Lambda}_d = \text{diag}(\lambda_1, \dots, \lambda_d)$. Each row $\mathbf{b}_j \in \mathbb{R}^d$ of $\mathbf{B}$ serves as a low-dimensional representation of feature $j$ in the principal subspace. The magnitude $\|\mathbf{b}_j\|^2 = \sum_{r=1}^d \lambda_r v_{jr}^2$ corresponds to the variance that feature $j$ contributes to the $d$-dimensional principal subspace, recovering the weighted sum of squared loadings familiar from factor analysis \cite{thurstone1947fa}. Crucially, redundancy between features is geometrically represented by the \textit{angular proximity} of their corresponding row vectors in this $d$-dimensional loading space: two features that are near-collinear in the data space will have nearly parallel row vectors in $\mathbf{B}$.

\subsection{The Clustering Trap in Static Selection}
Independent ranking methods assign a scalar score to each feature and select the top-$k$ features. Under extreme multicollinearity---where a dominant latent factor drives many features---criteria such as marginal variance and Laplacian Score \cite{he2005laplacian} can concentrate the selected set on that factor. This is the \textit{clustering trap}: the selected subset is redundant, and lower-variance but independent latent factors can be missed. VPOS addresses this failure mode by making the selection criterion adaptive: the score of each candidate is updated after every pivot to reflect covariance variation not yet covered by the selected loading rows.

\section{Methodology: Orthogonal Deflation}

\subsection{Information Draining via Null-Space Projection}
Unlike one-shot ranking and clustering approaches such as variance filtering and Principal Feature Analysis (PFA) \cite{lu2007pfa}, VPOS makes a feature's score conditional on preceding pivots. Upon selecting feature $i_t$ at iteration $t$, we extract its loading vector $\mathbf{u} = \mathbf{b}_{i_t}^{(t)} \in \mathbb{R}^d$. To remove the represented direction from the remaining candidates, we define the orthogonal projection operator onto the complement of $\text{span}(\mathbf{u})$:
\begin{equation}
\mathbf{P}_{\mathbf{u}}^{\perp} = \mathbf{I}_d - \frac{\mathbf{u} \mathbf{u}^\top}{\mathbf{u}^\top \mathbf{u}}
\end{equation}
The loading embedding is updated via right-multiplication (deflation along the component axis):
\begin{equation}
\mathbf{B}^{(t+1)} = \mathbf{B}^{(t)} \mathbf{P}_{\mathbf{u}}^{\perp}
\end{equation}
For any remaining candidate feature $j$, its updated representation is:
\begin{equation}
\mathbf{b}_j^{(t+1)} = \mathbf{b}_j^{(t)} - \frac{\mathbf{b}_j^{(t)\top} \mathbf{u}}{\|\mathbf{u}\|^2} \mathbf{u}
\end{equation}
which is the component of $\mathbf{b}_j^{(t)}$ orthogonal to $\mathbf{u}$. Features that were highly correlated with $i_t$ (i.e., nearly parallel to $\mathbf{u}$ in loading space) will have their residual representation collapse toward zero, sharply reducing their score in the next iteration. Features that were already orthogonal to $i_t$ are unaffected.

\subsection{Incremental Orthogonal Variance Score}
At step $t$, the algorithm greedily maximises the residual row-norm score:
\begin{equation}
i_t = \arg\max_{j \notin \mathcal{S}_{t-1}} \|\mathbf{b}_j^{(t)}\|^2
\end{equation}
This criterion selects the feature with the largest residual covariance-factor norm. For diagnostic purposes, the incremental orthogonal variance (IOV) score accumulates the pivot residuals after $k$ selections:
\begin{equation}
\text{IOV}(i_1,\ldots,i_k) = \sum_{t=1}^{k} \|\mathbf{b}_{i_t}^{(t)}\|^2.
\end{equation}
This sum is order-dependent and is not a set-volume objective. The associated squared volume is instead the product of squared pivot residuals:
Let $\mathbf{B}_{\mathcal{S}_t}$ contain the selected loading rows after $t$ steps. Gram--Schmidt factorisation gives
\begin{equation}
\det\!\left(\mathbf{B}_{\mathcal{S}_t}\mathbf{B}_{\mathcal{S}_t}^{\top}\right)
= \prod_{s=1}^{t}\left\|\mathbf{b}_{i_s}^{(s)}\right\|^2.
\end{equation}
Consequently, for a fixed previously selected set, choosing the largest residual norm maximises the one-step multiplicative increase in squared volume. The cumulative IOV score is an order-dependent diagnostic, whereas the determinant is the product of the pivot residuals. As with other greedy pivoting procedures, the sequence maximises the conditional gain at each step rather than solving the global maximum-volume subset problem.

\subsection{Algorithmic Implementation}
\begin{algorithm}
\caption{Variance-Preserving Orthogonal Selection (VPOS)}
\label{alg:vpos}
\begin{algorithmic}[1]
\REQUIRE Data $\mathbf{X} \in \mathbb{R}^{D \times N}$, output-feature budget $k$, retained PCA rank $d$ with $1\leq k\leq d$
\STATE $\mathbf{X} \gets \text{Centre}(\text{Scale}(\mathbf{X}))$ \hfill $\triangleright$ Apply the prespecified data-domain scaling rule
\STATE Compute $\mathbf{\Sigma} = \frac{1}{N-1} \mathbf{X} \mathbf{X}^\top \in \mathbb{R}^{D \times D}$
\STATE Solve eigendecomposition: $\mathbf{\Sigma} = \mathbf{V} \mathbf{\Lambda} \mathbf{V}^\top$, sort descending
\STATE Initialise $\mathbf{B}^{(1)} = \mathbf{V}_d \mathbf{\Lambda}_d^{1/2} \in \mathbb{R}^{D \times d}$
\STATE $\mathcal{S} \gets \emptyset$
\FOR{$t = 1$ to $k$}
    \STATE $w_j \gets \|\mathbf{b}_j^{(t)}\|^2$ for all $j \notin \mathcal{S}$
    \STATE $i_t \gets \arg\max_j\, w_j$
    \STATE $\mathcal{S} \gets \mathcal{S} \cup \{i_t\}$
    \STATE $\mathbf{u} \gets \mathbf{b}_{i_t}^{(t)}$
    \STATE $\mathbf{B}^{(t+1)} \gets \mathbf{B}^{(t)} \left( \mathbf{I}_d - \frac{\mathbf{u} \mathbf{u}^\top}{\|\mathbf{u}\|^2} \right)$
\ENDFOR
\RETURN $\mathcal{S}$
\end{algorithmic}
\end{algorithm}

The mathematical presentation uses features as rows. The implementation follows the conventional scikit-learn storage layout, with observations as rows and features as columns, and therefore forms the same loading matrix $\mathbf{B}$ from the transpose orientation.

\noindent\textbf{Role of $d$.} The hyperparameter $d$ controls which part of the covariance structure VPOS operates on, and its effect is not symmetric. When $d$ is too small, $\mathbf{B}$ is an overly compressed summary that conflates distinct latent directions; the deflation drains correlations inaccurately, and the selected features do not fully escape redundancy. When $d$ is too large — particularly when $d$ approaches $N$ or includes principal components (PCs) whose eigenvalues are near-zero — the loading vectors for those trailing components are dominated by estimation noise. Deflating along noise directions injects randomness into the scoring sequence, degrading selection quality. The practical sweet spot is $d$ near the data's intrinsic dimensionality — enough to represent the main covariance structure faithfully, but not so large as to include noise components.

\noindent\textbf{Selection rule for $d$.} We sweep a small set of candidates satisfying $k\leq d<\min(N,D)$ and select the value minimising held-out normalised reconstruction error (NRE) on validation observations. NRE is the reconstruction sum of squared errors divided by the sum of squared errors of the training-mean predictor. It is scale-stable and avoids selecting $d$ on the final test observations. The same validation split is used to tune baseline hyperparameters, while the test observations remain untouched.

\noindent\textbf{Complexity.} Forming and diagonalising a dense $D\times D$ covariance matrix costs $O(ND^2+D^3)$. In the implementation we instead compute the leading PCA factors by singular value decomposition (SVD); a randomised rank-$d$ SVD has leading cost approximately $O(NDd)$, with additional lower-order orthogonalisation costs depending on oversampling and power iterations \cite{halko2011finding}. Scoring and deflation cost $O(kDd)$.

\section{Related Work}

Since VPOS is an unsupervised method that selects discrete original features based on variance structure, we restrict this survey to methods in the same problem class. We exclude supervised approaches (wrapper and embedded methods) that require a target variable. We also exclude several important nonlinear feature extraction methods such as autoencoders, kernel PCA, and Uniform Manifold Approximation and Projection (UMAP), as they produce transformed representations rather than original feature subsets. Within unsupervised feature selection, methods fall into four families: spectral/graph-based, sparse PCA, reconstruction-driven, and matrix factorisation/deflation methods. We survey the most relevant prior work in each and clarify where VPOS sits.

\subsection{Spectral and Graph-Based Methods}
Early spectral methods evaluate features based on their ability to preserve local neighbourhood structure. Laplacian Score \cite{he2005laplacian} ranks features independently by consistency with a graph Laplacian. MCFS \cite{cai2010mcfs} instead learns sparse regressions from the complete feature matrix to a spectral embedding, so it permits multivariate competition among features. Joint $\ell_{2,1}$-norm methods \cite{nie2010efficient} likewise combine row-sparse regression and graph-based manifold learning. These methods target local geometry; unlike VPOS, they do not explicitly construct a sequence of pivots that orthogonally covers a truncated global covariance factor.

\subsection{Sparse PCA and Soft Selection}
A separate line of work achieves implicit feature selection through sparsity regularisation on PCA loadings. Sparse PCA (SPCA) by Zou et al.\ \cite{zou2006spca} reformulates PCA as a regression problem with least absolute shrinkage and selection operator (LASSO) penalties on the loadings, driving many loadings to zero. Jolliffe et al.\ \cite{jolliffe2003modified} introduced SCoTLASS, which directly penalises the $\ell_1$ norm of principal component coefficients. More recently, Yi et al.\ \cite{yi2019adaptive} proposed Adaptive Weighted Sparse PCA (AWSPCA), incorporating adaptive weights into the sparsity penalty to improve robustness to noise and outliers. Hu et al.\ \cite{hu2026spspca} introduce an interpretable sparse-PCA framework with a single equilibrium parameter that adaptively adjusts variable penalties across components.

In contrast, VPOS provides ``hard'' geometric selection rather than soft weighting. While SPCA and its variants produce sparse linear combinations (soft-scaling feature weights), VPOS selects a discrete index set $\mathcal{S}$ of original features. This is useful in interpretability-sensitive domains where mixing original variables is undesirable. Aside from the requested subset size $k$, VPOS has one tuning parameter, the retained rank $d$, which has a geometric interpretation and is analysed in Section~\ref{sec:sensitivity}.

\subsection{Reconstruction and Subspace Methods}
Reconstruction-driven selection minimises data loss in the selected subset. Principal Feature Analysis (PFA) by Lu et al.\ \cite{lu2007pfa} clusters PCA loading vectors using K-means and selects the feature closest to each cluster centroid. VPOS replaces this clustering stage with a deterministic sequence of orthogonal deflation steps that successively covers complementary loading directions. Methods based on joint $\ell_{2,1}$ norms \cite{nie2010efficient, hou2014joint} embed sparse regression within a graph Laplacian framework. Double-sparsity constrained methods \cite{xiu2025double} extend this by embedding sparsity constraints directly into the PCA framework.

The closest additional comparator is Principal Variable Selection (PVS) by Skogholt et al.\ \cite{skogholt2023pvs}. PVS ranks normalised residual data columns through a PCA variance-voting function and then deflates the sample-space matrix by modified Gram--Schmidt. VPOS differs in three explicit respects: it pivots the unnormalised residual rows of the truncated covariance factor $\mathbf{V}_d\mathbf{\Lambda}_d^{1/2}$; its residual score is the diagonal of the deflated rank-$d$ covariance approximation; and its PCA truncation $d$ is selected on validation reconstruction error. We therefore include truncated PVS at the same retained rank as a direct experimental baseline.

\subsection{Matrix Factorisation and Deflation Methods}
The Column Subset Selection Problem (CSSP) \cite{mahoney2011randomized} is the overarching problem domain for discrete feature selection from a matrix: given $\mathbf{A} \in \mathbb{R}^{m \times n}$, select a subset $\mathcal{S}$ of $k$ columns such that the residual $\|\mathbf{A} - \mathbf{A}_{\mathcal{S}}\mathbf{A}_{\mathcal{S}}^{+}\mathbf{A}\|_F$ is minimised, where $\mathbf{A}_{\mathcal{S}}$ is the submatrix of selected columns and $(\cdot)^+$ denotes the Moore-Penrose pseudoinverse. VPOS, along with all methods in this subsection, can be understood as algorithmic approaches to (variants of) CSSP. The key distinctions among them lie in the operating space, the deflation mechanism, and the weighting scheme.

The foundational Golub--Klema--Stewart (GKS) approach \cite{golub1976gks} applies pivoted QR to the right singular vectors $\mathbf{V}_k^\top$ to identify representative columns. Pivoted QR is algorithmically equivalent to Gram--Schmidt with column pivoting, so GKS has the same projection-and-pivot mechanics as VPOS. Recent work extends such selection to Bayesian experimental design \cite{eswar2024gks}. Boutsidis et al.\ \cite{boutsidis2009improved} analyse a complementary two-stage randomised/deterministic CSSP algorithm with matrix-approximation bounds.

GKS operates on the \textit{unweighted} right singular vectors $\mathbf{V}_k$ and targets matrix approximation quality. VPOS shares the same projection-and-pivot structure but applies it to the rows of $\mathbf{V}_d\mathbf{\Lambda}_d^{1/2}$, where the $\mathbf{\Lambda}_d^{1/2}$ weighting makes directions explaining more variance contribute proportionally more to the selection score. The key practical distinction is that VPOS combines this weighting with validation-selected truncation, which determines the covariance subspace in which the selection is performed and avoids noise-dominated trailing components.

Related studies establish the individual roles of weighted principal components and QR-based selection. Kim and Rattakorn \cite{kim2011weightedpc} rank features using weighted principal components, and Liu et al.\ \cite{liu2012wpcs} weight sparse-PC loadings by singular values before static ranking. Maung and Schweitzer \cite{maung2013pass} accelerate pivoted QR for unsupervised selection while preserving the raw-data pivot sequence, and Moslemi and Ahmadian \cite{moslemi2024rrqr} use strong rank-revealing QR (RRQR) directly and within nonnegative matrix factorisation (NMF) and genetic-algorithm variants. Moreover, because $\mathbf{B}\mathbf{B}^{\top}=\mathbf{V}_d\mathbf{\Lambda}_d\mathbf{V}_d^{\top}$, the same exact-arithmetic pivots may be viewed as diagonally pivoted Cholesky on the truncated covariance approximation \cite{harbrecht2012cholesky}. VPOS brings these strands together by pivoting the validation-truncated, variance-weighted covariance factor and by isolating the effects of weighting, truncation, and deflation experimentally.

Yaghooti et al.\ \cite{yaghooti2023gramschmidt} proposed Gram--Schmidt methods for unsupervised feature extraction and selection. Their framework constructs a series of covariance matrices $\Sigma_j$ by projecting the data onto orthonormalised function spaces, generalising PCA to detect linear and nonlinear redundancies. Their selector uses diagonal entries of iteratively deflated covariance matrices. VPOS instead performs sequential orthogonal deflation on the explicit $D\times d$ truncated covariance factor, providing an $O(Dd)$ update per selection when $d\ll D$ and a direct residual row-norm interpretation.

\textbf{How VPOS differs from standard CSSP algorithms.} The standard greedy CSSP algorithms (GKS, RRQR) apply Gram-Schmidt orthogonalisation directly to the columns of $\mathbf{A}$ or the unweighted singular vectors, selecting the column with the largest residual $\ell_2$ norm at each step after projecting out the already-selected subspace. While VPOS addresses the same CSSP problem domain — selecting a discrete subset of original features — three structural differences distinguish it from these classical algorithms:

\begin{enumerate}[leftmargin=*, label=(\roman*)]
    \item \textbf{Operating space.} Greedy CSSP on $\mathbf{X}$ operates directly on the raw data matrix, selecting features by residual norms in $\mathbb{R}^N$. GKS operates on the unweighted singular vectors $\mathbf{V}_k$. VPOS operates on the variance-weighted loading matrix $\mathbf{B} = \mathbf{V}_d\mathbf{\Lambda}_d^{1/2}$, which means features are scored by their contribution to the \textit{global variance structure} rather than by raw norms in sample space (vs.\ CSSP on $\mathbf{X}$) or unweighted geometric position (vs.\ GKS).

    \item \textbf{Variance weighting.} A greedy CSSP algorithm applied to the unweighted eigenvector matrix $\mathbf{V}_d$ (as in GKS) treats all principal components equally. VPOS weights eigenvector $r$ by $\sqrt{\lambda_r}$, so principal components explaining more variance receive proportionally more influence on the selection score $\|\mathbf{b}_j\|^2 = \sum_{r=1}^{d} \lambda_r v_{jr}^2$. This is the metric induced by the rank-$d$ covariance approximation.

    \item \textbf{Deflation space and cost.} Greedy CSSP on $\mathbf{X}$ updates $D$ feature rows in $\mathbb{R}^N$ at $O(DN)$ per pivot. VPOS updates $D$ loading rows in $\mathbb{R}^d$ at $O(Dd)$ per pivot. Thus the pivot-update arithmetic differs by a factor of $N/d$ when $d\ll N$; this comparison excludes preprocessing and factorisation costs.
\end{enumerate}

\textbf{Empirical comparison.} We compare VPOS directly against two CSSP-derived baselines: (i) \textbf{Greedy CSSP on $\mathbf{X}$}: applies Gram-Schmidt column selection to the preprocessed data matrix, selecting the feature with maximum residual norm in sample space at each step; (ii) \textbf{Greedy CSSP on $\mathbf{V}_d$}: applies the same greedy procedure to the unweighted eigenvector matrix, isolating the effect of the $\mathbf{\Lambda}_d^{1/2}$ variance weighting. Results are reported in Section~\ref{sec:cssp_results}.

\section{Theoretical Analysis}

\subsection{Rank Reduction and Subspace Coverage}

\begin{proposition}[Rank Reduction]
At each iteration $t$, the deflation $\mathbf{B}^{(t+1)} = \mathbf{B}^{(t)} \mathbf{P}_{\mathbf{u}}^{\perp}$ reduces the rank of $\mathbf{B}$ by exactly one:
\begin{equation}
\text{rank}(\mathbf{B}^{(t+1)}) = \text{rank}(\mathbf{B}^{(t)}) - 1
\end{equation}
provided $\mathbf{u} = \mathbf{b}_{i_t}^{(t)} \neq \mathbf{0}$.
\end{proposition}

\begin{proof}
Let $\mathcal{R}_t$ be the row space of $\mathbf{B}^{(t)}$. Since the nonzero pivot $\mathbf{u}$ is a row, $\mathbf{u}\in\mathcal{R}_t$. The restriction of $\mathbf{P}_{\mathbf{u}}^{\perp}$ to $\mathcal{R}_t$ has kernel $\mathcal{R}_t\cap\operatorname{span}(\mathbf{u})=\operatorname{span}(\mathbf{u})$, which is one-dimensional. Rank--nullity on this restriction therefore gives $\dim(\mathcal{R}_t\mathbf{P}_{\mathbf{u}}^{\perp})=\dim(\mathcal{R}_t)-1$, proving the claim.
\end{proof}

\begin{remark}
If the retained eigenvalues are positive, then $\operatorname{rank}(\mathbf{B})=d$ and after $d$ exact-arithmetic pivots the residual embedding is zero. VPOS therefore requires $k\leq d$ and terminates early if the largest residual norm falls below a numerical tolerance. The condition $D\geq d$ alone is insufficient: rank also depends on the sample rank and on the positivity of the retained eigenvalues.
\end{remark}

\subsection{Connection to Column-Pivoted QR and Determinant Growth}
Let $\mathbf{X}^{\top}=\mathbf{U}\mathbf{S}\mathbf{V}^{\top}$ be the SVD of the centred, preprocessed sample-by-feature matrix. Since $\mathbf{\Lambda}_d=\mathbf{S}_d^2/(N-1)$,
\begin{equation}
\mathbf{B}^{\top}=\frac{1}{\sqrt{N-1}}\mathbf{S}_d\mathbf{V}_d^{\top}.
\end{equation}
The VPOS pivot sequence therefore coincides, in exact arithmetic, with the sequence produced by column-pivoted QR (CPQR) applied to $\mathbf{B}^{\top}$. It can equivalently be viewed as greedy column selection on the rank-$d$ PCA reconstruction expressed in its orthonormal score coordinates, because left multiplication by $\mathbf{U}_d$ preserves column inner products. This connection places VPOS within the broader family of CPQR-based selection methods, while its practical contribution lies in the specific choice of the variance-weighted, validation-truncated operating matrix.

There is a second equivalent view. Since
\begin{equation}
\mathbf{B}\mathbf{B}^{\top}=\mathbf{V}_d\mathbf{\Lambda}_d\mathbf{V}_d^{\top}=\mathbf{\Sigma}_d,
\end{equation}
the CPQR residual squared norms are the residual diagonal entries generated by diagonally pivoted Cholesky on $\mathbf{\Sigma}_d$. Consequently, both procedures select the same pivot sequence in exact arithmetic when ties are resolved consistently. This provides an equivalent covariance-matrix interpretation of the VPOS selection scores and pivot sequence.

Together, these connections show that VPOS inherits the determinant-growth property of CPQR: at each step, the selected feature maximises the conditional determinant gain within the variance-weighted loading embedding. Because CPQR is a greedy method for the maximum-volume subset problem, the determinant interpretation applies to the one-step gain rather than guaranteeing a globally optimal subset. The practical significance is that these formal properties hold specifically because VPOS operates on the variance-weighted, validation-truncated factor $\mathbf{B}$ rather than on the raw sample matrix or unweighted eigenvectors.

\subsection{Geometry of Orthogonal Deflation}
The deflation step has an intuitive geometric interpretation. After selecting feature $i_t$ with loading direction $\mathbf{u}$, the updated loading of any remaining feature $j$ is its component orthogonal to $\mathbf{u}$:
\begin{equation}
\mathbf{b}_j^{(t+1)} = \mathbf{b}_j^{(t)} - \underbrace{\left(\frac{\mathbf{b}_j^{(t)\top} \mathbf{u}}{\|\mathbf{u}\|^2}\right)}_{\text{similarity to }i_t} \mathbf{u}
\end{equation}
The residual norm $\|\mathbf{b}_j^{(t+1)}\|^2 = \|\mathbf{b}_j^{(t)}\|^2 - (\mathbf{b}_j^{(t)\top}\hat{\mathbf{u}})^2$ where $\hat{\mathbf{u}} = \mathbf{u}/\|\mathbf{u}\|$ is the unit loading direction. Features that are nearly co-directional with $i_t$ in loading space have their score collapsed toward zero, removing redundant directions and preventing the clustering trap.

\subsection{Complexity and Scalability}
The selection phase costs $O(kDd)$. The decomposition cost depends on the PCA solver. A truncated randomised SVD has leading cost approximately $O(NDd)$, whereas covariance eigendecomposition or a full SVD has different time and memory requirements. The experiments use scikit-learn's automatic solver selection, and the reported wall-clock values therefore measure the complete implemented pipeline. Evaluation at still larger scales remains an important direction for establishing memory and runtime behaviour beyond the present benchmarks.

\section{Experimental Results and Discussion}

\subsection{Datasets and Evaluation Protocol}
\label{sec:setup}
We evaluate on six public real-world datasets and two controlled datasets. Table~\ref{tab:datasets} summarises their scale, class structure, and feature-selection budget. Class labels are used only for stratified splitting and downstream accuracy measurement; neither VPOS nor any unsupervised baseline receives them during feature selection.

\begin{table*}[!htbp]
\centering
\footnotesize
\setlength{\tabcolsep}{5pt}
\caption{Dataset characteristics. $N$ is the number of observations, $D$ the total number of input features, and $k$ the number of original features returned by each selector.}
\label{tab:datasets}
\begin{tabular}{lrrrlp{6.0cm}}
\toprule
Dataset & $N$ & $D$ & Classes & $k$ & Evaluation role \\
\midrule
Wine & 178 & 13 & 3 & 4 & Low-dimensional tabular data with moderate feature correlation \\
Breast Cancer & 569 & 30 & 2 & 6 & Correlated groups of cell-nucleus measurements \\
Digits & 1,797 & 64 & 10 & 16 & Small images with local spatial correlation \\
MNIST & 70,000 & 784 & 10 & 50 & Large grayscale-image scalability benchmark \\
CIFAR-10-small & 20,000 & 3,072 & 10 & 100 & High-dimensional colour-image scalability benchmark \\
Colon Cancer & 62 & 2,000 & 2 & 20 & Gene expression with $D\gg N$ and rank-deficient covariance \\
HighDim & 150 & 500 & 3 & 20 & Controlled informative, redundant, and noise variables with $D>N$ \\
Multi-Sector & 1,000 & 12 & -- & 3 & Controlled block correlation with three independent latent factors \\
\bottomrule
\end{tabular}
\end{table*}

\noindent\textbf{Small tabular datasets.} Wine contains 13 chemical measurements for three wine cultivars. Breast Cancer Wisconsin contains 30 measurements derived from cell-nucleus images for benign and malignant diagnoses; related mean, standard-error, and worst-case measurements create correlated feature groups. Digits represents each handwritten digit by 64 intensities on an $8\times8$ grid. Together, these datasets test whether the method remains useful when the feature space is modest and covariance can be estimated relatively well.

\noindent\textbf{Large image datasets.} Modified National Institute of Standards and Technology (MNIST) \cite{lecun1998mnist} contains 70,000 handwritten digits represented by 784 grayscale pixels. Canadian Institute for Advanced Research 10-class small (CIFAR-10-small) \cite{krizhevsky2009cifar} contains 20,000 colour images represented by 3,072 raw red-green-blue pixel values. These datasets test computational scaling and selection under extensive spatial and channel correlation. The stated dataset sizes are used without further experimental subsampling.

\noindent\textbf{Small-sample high-dimensional data.} Colon Cancer \cite{alon1999colon} contains expression measurements for 2,000 genes from 62 tissue samples (40 tumour and 22 normal). It provides a real $D\gg N$ setting in which covariance is necessarily rank-deficient and correlated genes may reflect shared biological processes. HighDim complements it with a controlled construction of 500 variables from 150 observations and three classes: 30 variables are informative, 200 are linear combinations of informative variables, and the remainder are noise. The controlled construction separates the effect of redundancy from domain-specific biological structure.

\noindent\textbf{Block-correlation stress test.} Multi-Sector contains 12 variables arranged into three groups of four. Each group shares an independent Gaussian latent factor, with independent Gaussian noise of standard deviation 0.1 added to each variable. This produces within-group correlations near 0.99 and near-zero correlations across groups. With a budget of three features, an effective diversity-aware selector should represent all three groups rather than repeatedly select from the same dominant group. Multi-Sector has no class labels, so it is evaluated through group coverage and reconstruction only.

MNIST, CIFAR-10-small, and Colon Cancer are obtained from OpenML; Wine, Breast Cancer, and Digits are loaded through scikit-learn. The two controlled datasets are generated with random seed 42.

The evaluation prevents observations used for final scoring from influencing scaling, selection, or hyperparameter choice. Wine, Breast Cancer, Digits, Colon Cancer, and HighDim use nested stratified five-fold evaluation; MNIST and CIFAR-10 use the same nested design with three outer folds to control computational cost. Within each outer training fold, an inner 75/25 split selects $d$, the Laplacian neighbourhood size, and the MCFS penalty by validation normalised reconstruction error (NRE). The selected settings are then refitted on the complete outer training fold and evaluated on its untouched outer test fold. Multi-Sector uses a random holdout because its observations are independently generated. All VPOS candidates satisfy $k\leq d<\min(N,D)$.

For selected indices $\mathcal{S}$, an ordinary least-squares (OLS) model is fitted on training observations to reconstruct the target features from $\mathbf{X}_{\mathcal{S}}$. We report held-out
\begin{equation}
\operatorname{NRE}=\frac{\|\mathbf{X}_{\mathrm{test}}-\widehat{\mathbf{X}}_{\mathrm{test}}\|_F^2}
{\|\mathbf{X}_{\mathrm{test}}-\bar{\mathbf{X}}_{\mathrm{train}}\|_F^2},
\end{equation}
where $\bar{\mathbf{X}}_{\mathrm{train}}$ denotes the training feature means repeated across the test observations and lower is better. An NRE of zero denotes exact reconstruction, an NRE of one means that reconstruction from the selected features is no better than predicting the training mean, and values above one are worse than that mean-prediction baseline. For CIFAR-10 and Colon Cancer, the same seeded sample of 500 target features is used for all methods. Five-nearest-neighbour (5-NN) accuracy is the fraction of correctly classified test observations, so higher is better. It is a downstream diagnostic and is not used by any selector. Non-image variables are standardised using means and standard deviations fitted only on the relevant training partition. MNIST and CIFAR-10 pixels already share a common intensity scale and are divided by 255; preserving their relative pixel variances avoids amplifying rare border activity. The marginal-variance baseline is computed before unit-variance scaling on the non-image datasets.

We compare against marginal variance, PCA weighted-loading norms without deflation, Laplacian Score \cite{he2005laplacian}, MCFS \cite{cai2010mcfs}, PFA \cite{lu2007pfa}, PVS \cite{skogholt2023pvs}, and greedy orthogonal selection on the preprocessed sample matrix (CSSP-$X$). The PCA/no-deflation, unweighted-eigenvector, PFA, and truncated-PVS comparisons use the same fold-specific validation-selected $d$ as VPOS; this matched-rank design isolates the selection mechanism rather than independently optimising each baseline's truncation. Laplacian neighbourhood size and MCFS penalty are tuned separately on the same inner validation observations. Following the MCFS formulation, every spectral coordinate is regressed on the complete feature matrix.

Experiments were run on an Apple M1 Pro MacBook Pro with 10 central processing unit (CPU) cores and 32\,GB memory. Timings measure the final feature-selection call in each outer fold and exclude data loading, splitting, hyperparameter tuning, preprocessing, and downstream evaluation. They include computations intrinsic to a selector, such as PCA or graph construction, and are reported as within-machine mean $\pm$ standard deviation across outer folds. Parallelism follows the default behaviour of the underlying numerical libraries.

\subsection{Main Results}
Table~\ref{tab:summary} reports the complete comparison with all eight methods. Lower NRE and higher 5-NN accuracy indicate better performance, and bold type marks the best observed mean for each dataset and metric. All values are mean $\pm$ standard deviation across outer test folds. Table~\ref{tab:runtime} separately reports selection time on the two large-image datasets.

\begin{table*}[!htbp]
\centering
\scriptsize
\setlength{\tabcolsep}{6pt}
\renewcommand{\arraystretch}{0.88}
\caption{Complete held-out comparison. Lower NRE is better; higher 5-NN accuracy is better. Bold marks the best mean for each dataset. Values are mean $\pm$ standard deviation over five outer folds for Wine, Breast Cancer, Digits, Colon Cancer, and HighDim and three outer folds for MNIST and CIFAR-10. PCA denotes weighted PCA loading scores without deflation; CSSP-$X$ denotes greedy orthogonal selection on the preprocessed sample matrix.}
\label{tab:summary}
\begin{tabular}{llcc}
\toprule
Dataset & Method & NRE $\downarrow$ & 5-NN accuracy $\uparrow$ \\
\midrule
\multirow{8}{*}{Wine}
 & Variance & .4817$\pm$.0402 & .8932$\pm$.0308 \\
 & PCA/no deflation & .4817$\pm$.0398 & .8427$\pm$.0317 \\
 & Laplacian Score & .4836$\pm$.0254 & \textbf{.9438$\pm$.0282} \\
 & MCFS & .4450$\pm$.0288 & .9378$\pm$.0377 \\
 & PFA & .4033$\pm$.0267 & .9329$\pm$.0462 \\
 & PVS & .4055$\pm$.0434 & .9159$\pm$.0337 \\
 & CSSP-$X$ & .4444$\pm$.0545 & .8595$\pm$.0792 \\
 & VPOS & \textbf{.4020$\pm$.0302} & .9324$\pm$.0257 \\
\midrule
\multirow{8}{*}{Breast Cancer}
 & Variance & .4523$\pm$.0686 & .9456$\pm$.0258 \\
 & PCA/no deflation & .3834$\pm$.0922 & .9438$\pm$.0260 \\
 & Laplacian Score & .4093$\pm$.0600 & .9438$\pm$.0325 \\
 & MCFS & .3148$\pm$.0347 & .9456$\pm$.0226 \\
 & PFA & .2347$\pm$.0511 & .9262$\pm$.0326 \\
 & PVS & .2324$\pm$.0413 & .9456$\pm$.0199 \\
 & CSSP-$X$ & .2780$\pm$.0416 & .8964$\pm$.0426 \\
 & VPOS & \textbf{.2022$\pm$.0318} & \textbf{.9526$\pm$.0211} \\
\midrule
\multirow{8}{*}{Digits}
 & Variance & .5532$\pm$.0635 & \textbf{.9566$\pm$.0140} \\
 & PCA/no deflation & .5994$\pm$.1309 & .7474$\pm$.1214 \\
 & Laplacian Score & .6008$\pm$.0685 & .7980$\pm$.0426 \\
 & MCFS & .5065$\pm$.0801 & .8714$\pm$.0335 \\
 & PFA & .5372$\pm$.0638 & .8119$\pm$.0640 \\
 & PVS & \textbf{.4829$\pm$.0667} & .8854$\pm$.0172 \\
 & CSSP-$X$ & .5044$\pm$.0760 & .7462$\pm$.0411 \\
 & VPOS & .5155$\pm$.0499 & .8387$\pm$.0558 \\
\midrule
\multirow{8}{*}{MNIST}
 & Variance & .4927$\pm$.0025 & .8667$\pm$.0027 \\
 & PCA/no deflation & .4848$\pm$.0038 & .8651$\pm$.0027 \\
 & Laplacian Score & 1.0000$\pm$.0000 & .1079$\pm$.0080 \\
 & MCFS & .6752$\pm$.0208 & .6788$\pm$.0098 \\
 & PFA & .2946$\pm$.0103 & .9405$\pm$.0054 \\
 & PVS & .2728$\pm$.0009 & .9476$\pm$.0026 \\
 & CSSP-$X$ & .2827$\pm$.0028 & \textbf{.9510$\pm$.0020} \\
 & VPOS & \textbf{.2723$\pm$.0017} & .9488$\pm$.0014 \\
\midrule
\multirow{8}{*}{CIFAR-10}
 & Variance & .7279$\pm$.0023 & .1871$\pm$.0024 \\
 & PCA/no deflation & .7279$\pm$.0023 & .1860$\pm$.0022 \\
 & Laplacian Score & .7291$\pm$.0012 & .1856$\pm$.0047 \\
 & MCFS & .5679$\pm$.0040 & .2185$\pm$.0038 \\
 & PFA & .2305$\pm$.0033 & .2608$\pm$.0014 \\
 & PVS & \textbf{.1606$\pm$.0036} & \textbf{.2990$\pm$.0013} \\
 & CSSP-$X$ & .1693$\pm$.0043 & .2966$\pm$.0098 \\
 & VPOS & .1633$\pm$.0009 & .2933$\pm$.0014 \\
\midrule
\multirow{8}{*}{Colon Cancer}
 & Variance & .8973$\pm$.0354 & .7244$\pm$.0963 \\
 & PCA/no deflation & .8917$\pm$.0798 & \textbf{.8538$\pm$.0707} \\
 & Laplacian Score & 1.0900$\pm$.1920 & .8244$\pm$.1006 \\
 & MCFS & 1.0149$\pm$.1613 & .6910$\pm$.0960 \\
 & PFA & \textbf{.7946$\pm$.0877} & .7077$\pm$.1286 \\
 & PVS & .8549$\pm$.0586 & .8077$\pm$.0875 \\
 & CSSP-$X$ & .8260$\pm$.0511 & .6603$\pm$.1092 \\
 & VPOS & .8269$\pm$.0512 & .6782$\pm$.0966 \\
\midrule
\multirow{8}{*}{HighDim}
 & Variance & .7783$\pm$.0234 & .5133$\pm$.0767 \\
 & PCA/no deflation & .7748$\pm$.0182 & .5400$\pm$.0596 \\
 & Laplacian Score & .7844$\pm$.0199 & .5933$\pm$.0983 \\
 & MCFS & .8153$\pm$.0166 & \textbf{.6267$\pm$.1256} \\
 & PFA & 1.0624$\pm$.0743 & .3933$\pm$.0494 \\
 & PVS & \textbf{.7638$\pm$.0124} & .5067$\pm$.1011 \\
 & CSSP-$X$ & 1.0620$\pm$.0161 & .3200$\pm$.0606 \\
 & VPOS & .7659$\pm$.0176 & .5533$\pm$.1121 \\
\bottomrule
\end{tabular}
\renewcommand{\arraystretch}{1}
\end{table*}

\begin{table}[!htbp]
\centering
\footnotesize
\caption{Feature-selection wall-clock time in seconds on the large-image datasets; lower is better. Values are mean $\pm$ standard deviation across three outer folds and include computations intrinsic to each selector.}
\label{tab:runtime}
\begin{tabular}{lrr}
\toprule
Method & MNIST & CIFAR-10 \\
\midrule
Variance & 0.08$\pm$0.02 & 0.10$\pm$0.03 \\
PCA/no deflation & 0.43$\pm$0.02 & 3.22$\pm$0.35 \\
Laplacian Score & 11.82$\pm$0.36 & 5.43$\pm$0.20 \\
MCFS & 12.74$\pm$0.25 & 9.67$\pm$0.49 \\
PFA & 1.83$\pm$0.12 & 7.52$\pm$0.33 \\
PVS & 23.23$\pm$0.66 & 83.07$\pm$1.51 \\
CSSP-$X$ & 4.55$\pm$0.08 & 10.25$\pm$0.15 \\
VPOS & 0.48$\pm$0.03 & 3.41$\pm$0.23 \\
\bottomrule
\end{tabular}
\end{table}

VPOS obtains the lowest mean NRE on Wine, Breast Cancer, and MNIST. On Wine it reaches 0.4020, narrowly improving on PFA at 0.4033; on Breast Cancer it reaches 0.2022 compared with 0.2324 for PVS. On MNIST, VPOS and PVS are nearly indistinguishable in reconstruction mean (0.2723 and 0.2728, respectively), with overlapping fold variation. VPOS also remains close to the lowest NRE on HighDim (0.3\% behind PVS) and CIFAR-10 (1.7\% behind PVS). PVS is the strongest reconstruction competitor on Digits, CIFAR-10, and HighDim.

Classification gives a different picture because VPOS uses no class labels and optimises covariance coverage rather than discrimination. VPOS has the highest 5-NN accuracy on Breast Cancer, while Laplacian Score leads on Wine, marginal variance on Digits, CSSP-$X$ on MNIST, PVS on CIFAR-10, PCA without deflation on Colon Cancer, and MCFS on HighDim. On MNIST, VPOS remains close to CSSP-$X$ in accuracy (0.9488 versus 0.9510). On Colon Cancer, the reconstruction differences among PFA, CSSP-$X$, and VPOS are small relative to the fold variability, while PCA without deflation has substantially higher classification accuracy. These results show that low reconstruction error does not necessarily imply high classification accuracy.

The large-scale timing comparison reveals a useful efficiency trade-off. On CIFAR-10, VPOS takes $3.41\pm0.23$\,s versus $83.07\pm1.51$\,s for PVS, $9.67\pm0.49$\,s for MCFS, $10.25\pm0.15$\,s for CSSP-$X$, and $7.52\pm0.33$\,s for PFA. VPOS is therefore about 24$\times$ faster than PVS with a 1.7\% mean NRE gap. On MNIST, VPOS takes $0.48\pm0.03$\,s versus $23.23\pm0.66$\,s for PVS and $12.74\pm0.25$\,s for MCFS. These measurements include PCA inside the VPOS call but remain machine-specific.

\subsection{Deflation and Weighting Ablations}
\label{sec:cssp_results}
Table~\ref{tab:ablation} isolates the contribution of sequential deflation by comparing VPOS with the same variance-weighted PCA representation ranked once without deflation. Lower values are better in both columns. VPOS improves on this matched no-deflation baseline on every labelled dataset, demonstrating a consistent held-out benefit from iterative orthogonal coverage. The relative NRE reduction ranges from 1.1\% on HighDim to 77.6\% on CIFAR-10.

\begin{table}[!htbp]
\centering
\footnotesize
\caption{Held-out mean NRE ablation; lower is better. Reduction is the percentage decrease from PCA/no deflation to VPOS across the same outer folds.}
\label{tab:ablation}
\begin{tabular}{lccc}
\toprule
Dataset & PCA/no deflation & VPOS & Reduction \\
\midrule
Wine & 0.4817 & \textbf{0.4020} & 16.5\% \\
Breast Cancer & 0.3834 & \textbf{0.2022} & 47.3\% \\
Digits & 0.5994 & \textbf{0.5155} & 14.0\% \\
MNIST & 0.4848 & \textbf{0.2723} & 43.8\% \\
CIFAR-10 & 0.7279 & \textbf{0.1633} & 77.6\% \\
Colon Cancer & 0.8917 & \textbf{0.8269} & 7.3\% \\
HighDim & 0.7748 & \textbf{0.7659} & 1.1\% \\
\bottomrule
\end{tabular}
\end{table}

On the three tabular datasets used for the operating-space ablation, VPOS versus unweighted orthogonal selection on $\mathbf{V}_d$ gives NRE 0.4020 versus 0.4556 (Wine), 0.2022 versus 0.2274 (Breast Cancer), and 0.5155 versus 0.5436 (Digits). Thus $\mathbf{\Lambda}_d^{1/2}$ weighting improves held-out reconstruction in all three. VPOS also gives the highest 5-NN accuracy among the three greedy-deflation variants on each dataset. Compared with greedy selection on raw $\mathbf{X}$, VPOS improves NRE on Wine and Breast Cancer, while raw-data selection gives the lower NRE on Digits (0.5044 versus 0.5155). The effect of truncation is therefore dataset-dependent, whereas the weighting effect is consistent across this ablation.

\subsection{Sensitivity to $d$ and Synthetic Structure}
\label{sec:sensitivity}
Sensitivity to $d$ is evaluated on validation observations. Every candidate satisfies $k\leq d<\min(N,D)$, ensuring that the loading embedding can support the requested number of positive-rank pivots. The final test observations remain independent of the choice of $d$.

The selected $d$ values, in outer-fold order, are Wine $(5,4,4,4,5)$, Breast Cancer $(6,6,6,6,6)$, Digits $(24,50,16,16,16)$, Colon Cancer $(20,22,34,20,22)$, HighDim $(20,20,20,20,20)$, MNIST $(50,75,50)$, and CIFAR-10 $(150,120,120)$. Thus $d$ is a fold-specific operating rank selected from the available training observations rather than a universally optimal dataset constant. Variation across folds, particularly for Digits and Colon Cancer, reflects the sensitivity of the estimated covariance structure to the available training observations rather than access to outer-test performance.

The validation-only diagnostic sweep also illustrates the dataset-dependent role of truncation. Wine has essentially identical validation NRE at $d=4$ and $d=5$ but degrades at larger values, while Digits is best at its smallest admissible value, $d=16$. HighDim deteriorates as $d$ approaches the sample-rank limit, and Colon Cancer has a noisy validation curve because its validation partitions contain very few observations. Under fixed intensity scaling, the image curves are flatter: MNIST validation NRE ranges from 0.2690 to 0.2759 over $d=50$--238, and CIFAR-10 ranges from 0.1575 to 0.1629 over $d=100$--300. Their fold-specific choices nevertheless show that the validation minimum can move among nearby candidates. These patterns support selecting $d$ independently inside each outer training fold.

On Multi-Sector, VPOS covers all three sectors and obtains test NRE 0.0151. PFA and PVS also cover all sectors with NRE 0.0147, while raw-data CSSP obtains 0.0144. PCA without deflation covers two sectors (0.3342), and Variance and Laplacian Score each cover one (approximately 0.67). The experiment demonstrates that adaptive orthogonal selection successfully prevents the budget from collapsing onto a single correlated block; the comparable results of PFA, PVS, and raw-data CSSP further show the value of explicitly diversity-aware baselines on this structure.

\subsection{Discussion}
The results identify three broad performance regimes for the VPOS deflation mechanism. The strongest improvements over matched PCA loading scores occur when substantial correlation leaves room for sequential coverage to diversify the selected subset. Deflation reduces NRE by 47.3\% on Breast Cancer, 43.8\% on MNIST, and 77.6\% on CIFAR-10. The image results are consistent with extensive spatial and channel redundancy, while the correlated measurement groups in Breast Cancer provide a tabular example. The Multi-Sector experiment makes the same mechanism directly observable: VPOS assigns its three-feature budget across all three independent blocks, whereas static PCA loading scores cover only two.

The benefit is moderate on Wine, Digits, and Colon Cancer, where deflation reduces NRE by 16.5\%, 14.0\%, and 7.3\%, respectively. VPOS achieves the lowest reconstruction error on Wine, while PVS and PFA lead on Digits and Colon Cancer. These datasets show that removing repeated covariance directions remains useful even when the resulting VPOS subset is not the best among all selection strategies.

The matched-deflation gain is small only on HighDim (1.1\%), suggesting that its static variance-weighted loading representation already captures much of the reconstructive information available at the selected rank. VPOS nevertheless remains within 0.3\% of the leading NRE there. On MNIST, VPOS has the lowest mean NRE and is close to both PVS in reconstruction and CSSP-$X$ in classification. Together with the sensitivity results, this indicates that VPOS is most effective when covariance is concentrated in repeated or block-correlated directions and when $d$ captures those directions without extending deeply into the noisy tail of the spectrum.

The operating-space ablation further separates the effects of the design choices. Variance weighting improves reconstruction over unweighted deflation on all three tabular datasets, while truncation improves over raw-data selection for Wine and Breast Cancer but not Digits. Thus weighting provides a consistent emphasis on high-variance directions in this comparison, whereas the value of truncation depends on how well the leading PCA subspace separates signal from discarded variation.

VPOS is particularly relevant when the retained variables must remain identifiable---for example, genes, clinical measurements, assets, or sensor channels---and the objective is to preserve global covariance structure efficiently. Its unsupervised selections can be reused across exploratory analyses or downstream tasks without reference to a particular labelling. The classification results should therefore be read as a complementary diagnostic: covariance preservation and class discrimination may align, as on Breast Cancer, or favour different subsets, as on Colon Cancer and HighDim.

\section{Conclusion}
This study introduced VPOS, an unsupervised feature-selection method that combines a truncated, variance-weighted PCA representation with sequential orthogonal deflation. VPOS selects features by iteratively choosing the row of $\mathbf{V}_d\mathbf{\Lambda}_d^{1/2}$ with the largest residual norm and projecting out its direction, so the returned original variables cover complementary directions of the retained covariance structure. The analysis establishes rank-reduction guarantees, a formal connection to column-pivoted QR, and the rank condition $k\leq d$. Selecting $d$ on validation observations provides a data-dependent operating rank while keeping the final test observations independent of model and hyperparameter selection.

Experiments on six public datasets and two controlled datasets demonstrate the effect of this construction across tabular, image, gene-expression, and $D\gg N$ settings. Relative to matched PCA loading scores without deflation, VPOS reduces held-out NRE on every labelled benchmark, with improvements ranging from 1.1\% to 77.6\%. It achieves the lowest mean reconstruction error on Wine, Breast Cancer, and MNIST and remains within 1.7\% of the leading method on CIFAR-10 and HighDim. On the Multi-Sector stress test, VPOS selects features from all three independent blocks and attains an NRE of 0.0151. The large-image experiments also show the computational advantage of pivoting in the truncated loading space: on CIFAR-10, VPOS is approximately 24 times faster than PVS while retaining a mean reconstruction error within 1.7\% of it. Downstream classification performance varies across datasets, confirming that unsupervised covariance preservation and class discrimination capture related but distinct properties of a selected subset.

Taken together, the results support VPOS as an efficient and interpretable approach when the objective is to retain a small set of original variables that represents global linear covariance structure. Future work may extend the framework through regularised covariance estimation for very small samples, nonlinear loading representations, stability analysis across resamples, and sharper reconstruction bounds as functions of $k$ and $d$. Evaluation on gene-expression collections substantially larger than the 62-sample Colon Cancer dataset, and on learned feature representations rather than raw image pixels, would further establish its behaviour in modern high-dimensional applications.

\section*{Declaration of competing interest}
The authors declare that they have no known competing financial interests or personal relationships that could have appeared to influence the work reported in this paper.

\section*{Data availability}
All public datasets used in this study are available through scikit-learn or OpenML. The two controlled datasets are generated as described in Section~\ref{sec:setup} using random seed 42.

\end{document}